\documentclass[12pt]{article}

\usepackage[utf8]{inputenc} 
\usepackage{titling}
\usepackage{xcolor}
\usepackage{amsmath}
\usepackage{amssymb}
\usepackage{amsthm}
\usepackage[font=footnotesize,labelfont=bf]{caption}
\usepackage{anyfontsize}
\usepackage[font=footnotesize]{subcaption}
 
\newcommand{\eps}{\varepsilon}

\newcommand{\RR}{\mathbb{R}}

\newcommand{\EE}{\mathbb{E}}

\newcommand{\Wcal}{\mathcal{W}}
\newcommand{\Dcal}{\mathcal{D}}
\newcommand{\Lip}{\mathrm{Lip}_{1}(X)}
\newcommand{\Lipxy}{\mathrm{Lip}_{1}(X{\times}Y)}
\newcommand{\Gcal}{\mathcal{G}}

\newtheorem{theorem}{Theorem}
\newtheorem{lemma}[theorem]{Lemma}
\newtheorem{assumption}[theorem]{Assumption}

\pretitle{
  \begin{center}
    \Large\bfseries
  \noindent\vrule height 1.5pt width \textwidth 
  \vskip .75em plus .25em minus .25em
}
\posttitle{%
  \end{center}%
  \noindent\vrule height 2.0pt width \textwidth
  \vskip .75em plus .25em minus .25em
}

\begin{document}
\title{About exchanging expectation and supremum for conditional Wasserstein GANs}

\author{J\"org Martin\thanks{joerg.martin(at)ptb.de, Physikalisch-Technische Bundesanstalt (PTB), Abbestraße 2, 10587 Berlin, Germany.}}\date{}
\maketitle
\vspace{-2cm}
\begin{abstract}
In cases where a Wasserstein GAN depends on a condition the latter is usually handled via an expectation within the loss function.
Depending on the way this is motivated, the discriminator is either required to be Lipschitz-1 in both or in only one of its arguments. For the weaker requirement to become usable one needs to exchange a supremum and an expectation. This is a mathematically perilous operation, which is, so far, only partially justified in the literature. This short mathematical note intends to fill this gap and provides the mathematical rationale for discriminators that are only partially Lipschitz-1 for cases where this approach is more appropriate or successful. 
\end{abstract}

Suppose our training data follows a distribution with the density $\pi(x,y)$ on $X\times Y$, where $X=\RR^{n_X}$, $Y=\RR^{n_Y}$.
\footnote{The assumption that $X$ and $Y$ are finite-dimensional is mostly for the sake of presentation, the argument can easily be extended to a Banach space setting as in \cite{adler2018banach}.} 
We assume that for each $y\in Y$ there is a probability distribution $\Gcal(y)$ on $X$ that can be sampled from via a measurable map $G(z,y),z\sim\eta$, where $\eta$ denotes the distribution on an arbitrary probability space. 
In terms of generative adversarial networks (GANs) the mapping $G$ is typically known as the generator \cite{goodfellow2020generative}, is implemented as a neural network and is used to approximate a target distribution. In contrast to the original GAN framework, $\Gcal$ and $G$ are here assumed to depend on $y$, which is known as a conditional GAN (cGAN) \cite{mirza2014conditional}.

We want $\Gcal(y)$ to approximate the conditional distribution $\pi(x|y)$:
\begin{align}
	\label{eq:approximation}
	\Gcal(y) \approx \pi(x|y) = \frac{\pi(x,y)}{\pi(y)} \,.
\end{align}
If this approximation is done by minimizing the Wasserstein-1-distance , the term (conditional) Wasserstein GAN ((c)WGAN is used.
The concept of WGANs was first proposed in the seminal works \cite{arjovsky2017wasserstein,gulrajani2017improved}. A study of a conditional relation \eqref{eq:approximation} under the Wasserstein metric is done in various works in the literature for different applications, compare e.g. \cite{ebenezer2019single,luo2018eeg,jin2019image,patzelt2019conditional,adler2018deep,adler2019deep,ohayon2021high}. In these works there appear to be several philosophies on how to deal with the $y$-dependency in \eqref{eq:approximation}. One natural approach is the following: Sample $\tilde{y}\sim \pi(y)$ and then $\tilde{x}\sim \Gcal(\tilde{y})$ and call their distribution $\pi(\tilde{x}, \tilde{y})$. Then, minimize the following distance w.r.t $G$
\begin{align}
	\label{eq:Wxy}
	\Wcal(\pi(\tilde{x}, \tilde{y}),\pi(x,y)) = \!\!\!\!\! \sup_{D\in\Lipxy}\!\! \left(\EE_{x,y\sim\pi(x,y)}\left[D(x,y)\right]{-}\EE_{z\sim\eta, y\sim \pi(y)}\left[D(G(z,y),y)\right]\right)\,,
\end{align}
with $\Lipxy$ denoting the set of real-valued Lipschitz-1 functions on $X\times Y$ and where we used the Kantorovich-Rubinstein duality. Note that this requires the discriminator $D$ to be Lipschitz-1 in \emph{both of its arguments}. While the conception above seems to be, at least implicitly, behind many works on cWGANs in the literature (compare e.g. \cite{ebenezer2019single,luo2018eeg,jin2019image,patzelt2019conditional}), the Lipschitz constraint is often only enforced w.r.t the first argument. One might attempt to motivate this with the fact that the marginal $\pi(y)$ is identical for $\pi(\tilde{x}, \tilde{y})$ and $\pi(x,y)$, but to the best of the author's knowledge there is no exhaustive argument in the literature. 

A different line of reasoning for a partial Lipschitz continuity of $D$, which this note will follow, is given in \cite{adler2018deep,adler2019deep}, where instead of \eqref{eq:Wxy} the objective
\begin{align}
	\label{eq:averaged_CWGAN}
	\EE_{y\sim \pi(y)}[\Wcal(\Gcal(y),\pi(x|y))] 
\end{align}
is considered, from which the authors obtain
\begin{align}
	\label{eq:useful_CWGAN}
	\sup_{D\in\Dcal}\EE_{(x,y)\sim\pi(x,y)}\left[D(x,y)-\EE_{z\sim\eta}\left[D(G(z,y),y)\right]\right] \,,
\end{align}
where $\mathcal{D}$ denotes the set of real-valued (and measurable) maps defined on $X\times Y$ and Lipschitz-1 in $X$ - and not necessarily in $Y$. The objective \eqref{eq:useful_CWGAN} is in alignment with a discriminator that is Lipschitz-1 in $x$ only. However, to get from \eqref{eq:averaged_CWGAN} to \eqref{eq:useful_CWGAN} a supremum and an expectation have to be exchanged which might cause problems concerning measurability. This issue is not entirely resolved in \cite{adler2018deep,adler2019deep} and the authors need to assume that there are no problems in this regard. 

A workaround could be to assume that $\pi(y)$ is only supported on a discrete subset, on which measurability is trivial. For cases where $y$ describes continuous objects such as images or audio data this becomes a bit of a cheat, as this support would not only have to include all of the training data, but also any possible $y$ one might encounter during evalution of $\Gcal(y)$ after the training. This would thus require some sort of further artificial discretization (e.g. by machine epsilon). It seems however much more natural, satisfactory and intuitive to allow for a continuous $\pi(y)$ in these cases. The goal of this mathematical note is to demonstrate that this can be done and that, under under relatively mild conditions, there is a consistent way to get from \eqref{eq:averaged_CWGAN} to \eqref{eq:useful_CWGAN}. We thereby provide a mathematical rationale for using the object \eqref{eq:useful_CWGAN} without any need of artifical discretization.

\begin{assumption}
\label{ass:compact_support}
The marginal $\pi(y)$ has a compact support, say on $K_{Y}\subseteq Y$.
\end{assumption}
\begin{assumption}
\label{ass:W_continuity}
The mapping $y\mapsto\pi(x|y)$ is well-defined on $K_{Y}$ and continuous
w.r.t the Wasserstein metric.
\end{assumption}
\begin{assumption}
\label{ass:E_continuity}
The object $\EE_{z\sim\eta}\left[|G(z,y)-G(z,y')|\right]$ converges
to $0$ for $y\rightarrow y'$ uniformly for all $y,y'\in K_{Y}$. 
\end{assumption}

The assumptions above could probably be weakened. However, they appear to be not unreasonable. Assumption $3$, for instance, is true if $G$ is a neural network with Lipschitz continuous activations and $\eta$ is a Gaussian distribution or a distribution of bounded support. Assumption $1$ is often true in practice, but can be dropped provided that the continuity in Assumption \ref{ass:W_continuity} is uniform and the following lemma remains true.

\begin{lemma}
\label{lem:continuity}
The map
$F:K_{Y}\ni y\mapsto\Wcal(\Gcal(y),\pi(x|y))$
is uniformly continuous\,.
\end{lemma}
\begin{proof}
As $K_{Y}$ is compact (Assumption 1), it is enough to show that $F$
is continuous. Take $y,y'\in K_{y}$ and note that
\begin{align*}
F(y) & =\sup_{f\in\Lip}\left(\EE_{x\sim\pi(x|y)}\left[f(x)\right]-\EE_{z\sim\eta}\left[f(G(z,y)\right]\right)\\
 & \leq\Wcal(\pi(x|y),\pi(x|y'))+\EE_{z\sim\eta}\left[|G(z,y)-G(z,y')|\right]+F(y') \,.
\end{align*}
By symmetry we can exchange the roles of $y$ and $y'$ and arrive at
$
|F(y)-F(y')|\leq\Wcal(\pi(x|y),\pi(x|y'))+\EE_{z\sim\eta}\left[|G(z,y)-G(z,y')|\right]
$,
which converges to 0 as $y\rightarrow y'$ due to Assumption 2 and 3. 
\end{proof}
Assumptions \ref{ass:compact_support}-\ref{ass:E_continuity} were basically chosen to make Lemma \ref{lem:continuity} true as it is the key to proving the main statement of this note, which will be our next task. In essence, Lemma \ref{lem:continuity} allows us to reduce the average over $y\sim \pi(y)$ in \eqref{eq:averaged_CWGAN} to a finite sum up to an arbitrarily small error. On finite sets measurability is no issue, so that we can safely exchange the supremum and the expectation.
\begin{theorem}
\label{thm:equivalence}
Under the Assumptions \ref{ass:compact_support}-\ref{ass:E_continuity} we have
\begin{align*}
\EE_{y\sim\pi(y)}\left[\Wcal(\Gcal(y),\pi(x|y)\right]=\sup_{D\in\Dcal}\text{\ensuremath{\EE}}_{(x,y)\sim\pi(x,y)}\left[D(x,y)-\EE_{z\sim\eta}\left[D(G(z,y),y)\right]\right] \,.
\end{align*}
\end{theorem}
\begin{proof}
The inequality 
\begin{align*}
\EE_{y\sim\pi(y)}[F(y)] & =\EE_{y\sim\pi(y)}\left[\Wcal(\Gcal(y),\pi(x|y))\right]\\
 & \geq\sup_{D\in\Dcal}\text{\ensuremath{\EE}}_{(x,y)\sim\pi(x,y)}\left[D(x,y)-\EE_{z\sim\eta}\left[D(G(z,y),y)\right]\right]
\end{align*}
is simply a consequence of the Kantorovich-Rubinstein duality, cf. \cite{adler2018deep}. We can therefore prove the claim once we have shown that
for any $\eps>0$ 
\begin{align}
\EE_{y\sim\pi(y)}[F(y)]\leq\sup_{D\in\Dcal}\text{\ensuremath{\EE}}_{(x,y)\sim\pi(x,y)}\left[D(x,y)-\EE_{z\sim\eta}\left[D(G(z,y),y)\right]\right]+\eps\label{eq:DifficultDirection}\,.
\end{align}
Using the Lemma above and Assumption 2 and 3 we can pick $\delta_{\eps}>0$
such that $|F(y)-F(y')|\leq\frac{\eps}{4}$, $\Wcal(\pi(x|y),\pi(x|y'))\leq\frac{\eps}{4}$
and $\EE_{z\sim\eta}\left[|G(z,y)-G(z,y')|\right]\leq\frac{\eps}{4}$ whenever
$|y-y'|<\delta_{\eps}$ for any $y,y'\in K_{Y}$. The compact set
$K_{Y}$ can be covered with a finite numbers of disjoint boxes $C_{k}^{\eps}$
of diameter smaller than $\delta_{\eps}$. For each $k$ fix some $y_{k}^{\eps}\in C_k^{\eps} \cap K_{Y}$.
Writing $F^{\eps}(y)=\sum_{k}F(y_{k}^{\eps})1_{C_{k}^{\eps}}(y)$
with $F$ as above and with $1_{C_{^{k}}^{\eps}}(y)$ equal to 1 if
$y\in C_{k}^{\eps}$ and else 0, we have 
\begin{align}
\EE_{y\sim\pi(y)}\left[F(y)\right] & =\EE_{y\sim\pi(y)}\left[F(y)-F^{\eps}(y)\right]+\EE_{y\sim\pi(y)}\left[F^{\eps}(y)\right]\nonumber \\
 & \leq\frac{\eps}{4}+\EE_{y\sim\pi(y)}\left[F^{\eps}(y)\right]\label{eq:Bounding_by_F_eps}\,.
\end{align}
Using once more the Kantorovich-Rubinstein duality we can now pick for every $k$ a $f_{k}^{\eps}\in\Lip$ such that
\begin{align}
F(y_{k}^{\eps}) \leq \frac{\eps}{4} + \EE_{x\sim\pi(x|y_{k}^{\eps})}\left[f_{k}^{\eps}(x)\right]-\EE_{z\sim\eta}\left[f_{k}^{\eps}(G(z,y_{k}^{\eps})\right] \,. \label{eq:Fy_eps_k} 
\end{align}
Given $y\in C_{k}^{\eps} \cap K_Y$ we can use $\EE_{z\sim\eta}\left[|G(z,y)-G(z,y_{k}^{\eps})|\right]\leq\frac{\eps}{4}$
by choice of $\delta_{\eps}$ and bound (\ref{eq:Fy_eps_k}) further by 
\begin{align}
F(y_{k}^{\eps})\leq2\cdot\frac{\eps}{4}+\EE_{x\sim\pi(x|y_{k}^{\eps})}\left[f_{k}^{\eps}(x)\right]-\EE_{z\sim\eta}\left[f_{k}^{\eps}(G(z,y))\right]\label{eq:F_y_eps_k_y_bound}\,.
\end{align}
Combining (\ref{eq:Bounding_by_F_eps}) and (\ref{eq:F_y_eps_k_y_bound})
we end up with 
\begin{align}
\EE_{y\sim\pi(y)}\left[F(y)\right]
&\leq \frac{3\eps}{4}+\EE_{y\sim\pi(y)}\left[\sum_{k}1_{C_{k}^{\eps}}(y)\left(\EE_{x\sim\pi(x|y_{k}^{\eps})}\left[f_{k}^{\eps}(x)\right]-\EE_{z\sim\eta}\left[f_{k}^{\eps}(G(z,y))\right]\right)\right]\nonumber \\
&\leq\eps+\EE_{y\sim\pi(y)}\left[\sum_{k}1_{C_{k}^{\eps}}(y)\left(\EE_{x\sim\pi(x|y)}\left[f_{k}^{\eps}(x)\right]-\EE_{z\sim\eta}\left[f_{k}^{\eps}(G(z,y))\right]\right)\right]\,,\label{eq:total_F_bound}
\end{align}
where we used in the last step that $\Wcal(\pi(x|y),\pi(x|y_{k}^{\eps}))\leq\frac{\eps}{4}$.
Writing now 
$D^{\eps}(x,y)=\sum_{k}1_{C_{k}^{\eps}}(y)\cdot f_{k}^{\eps}(x)$
we have apparently $D^{\eps}\in\mathcal{D}$ and, moreover,
\begin{align*}
&\text{\ensuremath{\EE}}_{(x,y)\sim\pi(x,y)}\left[D^{\eps}(x,y)-\EE_{z\sim\eta}\left[D^{\eps}(G(z,y),y)\right]\right]\\
&=\EE_{y\sim\pi(y)}\left[\sum_{k}1_{C_{k}^{\eps}}(y)\left(\EE_{x\sim\pi(x|y)}\left[f_{k}^{\eps}(x)\right]-\EE_{z\sim\eta}\left[f_{k}^{\eps}(G(z,y))\right]\right)\right]\,,
\end{align*}
which is just the second term on the right hand side of (\ref{eq:total_F_bound}); this yields
\begin{align*}
\EE_{y\sim\pi(y)}\left[F(y)\right] & \leq\eps+\text{\ensuremath{\EE}}_{(x,y)\sim\pi(x,y)}\left[D^{\eps}(x,y)-\EE_{z\sim\eta}\left[D^{\eps}(G(z,y),y)\right]\right]\\
 & \leq\eps+\sup_{D\in\Dcal}\text{\ensuremath{\EE}}_{(x,y)\sim\pi(x,y)}\left[D(x,y)-\EE_{z\sim\eta}\left[D(G(z,y),y)\right]\right]\,,
\end{align*}
which shows (\ref{eq:DifficultDirection}). 
\end{proof}
\vskip 0.2in
\bibliography{NoteCWGAN}

\begin{thebibliography}{10}

\bibitem{adler2018banach}
{\sc Adler, J., and Lunz, S.}
\newblock Banach wasserstein {{GAN}}.
\newblock In {\em NeurIPS\/} (2018), pp.~6755--6764.

\bibitem{adler2018deep}
{\sc Adler, J., and {\"O}ktem, O.}
\newblock Deep bayesian inversion.
\newblock arXiv preprint arXiv:1811.05910, 2018.

\bibitem{adler2019deep}
{\sc Adler, J., and {\"O}ktem, O.}
\newblock Deep posterior sampling: Uncertainty quantification for large scale
  inverse problems.
\newblock In {\em International Conference on Medical Imaging with Deep
  Learning\/} (2019).

\bibitem{arjovsky2017wasserstein}
{\sc {Arjovsky}, M., {Chintala}, S., and {Bottou}, L.}
\newblock {Wasserstein {GAN}}.
\newblock arXiv e-prints, Jan. 2017.

\bibitem{ebenezer2019single}
{\sc Ebenezer, J.~P., Das, B., and Mukhopadhyay, S.}
\newblock Single image haze removal using conditional wasserstein generative
  adversarial networks.
\newblock In {\em 2019 27th European Signal Processing Conference (EUSIPCO)\/}
  (2019), IEEE, pp.~1--5.

\bibitem{goodfellow2020generative}
{\sc Goodfellow, I., Pouget-Abadie, J., Mirza, M., Xu, B., Warde-Farley, D.,
  Ozair, S., Courville, A., and Bengio, Y.}
\newblock Generative adversarial networks.
\newblock {\em Commun. ACM 63}, 11 (Oct. 2020), 139144.

\bibitem{gulrajani2017improved}
{\sc Gulrajani, I., Ahmed, F., Arjovsky, M., Dumoulin, V., and Courville,
  A.~C.}
\newblock Improved training of wasserstein {GANs}.
\newblock In {\em NeurIPS\/} (2017), I.~Guyon, U.~von Luxburg, S.~Bengio, H.~M.
  Wallach, R.~Fergus, S.~V.~N. Vishwanathan, and R.~Garnett, Eds.,
  pp.~5767--5777.

\bibitem{jin2019image}
{\sc Jin, Q., Luo, X., Shi, Y., and Kita, K.}
\newblock Image generation method based on improved condition {GAN}.
\newblock In {\em 2019 6th International Conference on Systems and Informatics
  (ICSAI)\/} (2019), IEEE, pp.~1290--1294.

\bibitem{luo2018eeg}
{\sc Luo, Y., and Lu, B.-L.}
\newblock Eeg data augmentation for emotion recognition using a conditional
  wasserstein {GAN}.
\newblock In {\em 2018 40th Annual International Conference of the IEEE
  Engineering in Medicine and Biology Society (EMBC)\/} (2018), IEEE,
  pp.~2535--2538.

\bibitem{mirza2014conditional}
{\sc Mirza, M., and Osindero, S.}
\newblock Conditional generative adversarial nets.
\newblock arXiv preprint arXiv:1411.1784, 2014.

\bibitem{ohayon2021high}
{\sc Ohayon, G., Adrai, T., Vaksman, G., Elad, M., and Milanfar, P.}
\newblock High perceptual quality image denoising with a posterior sampling
  cgan.
\newblock {\em arXiv preprint arXiv:2103.04192\/} (2021).

\bibitem{patzelt2019conditional}
{\sc Patzelt, F., Haschke, R., and Ritter, H.~J.}
\newblock Conditional {WGAN} for grasp generation.
\newblock In {\em ESANN\/} (2019).

\end{thebibliography}
\bibliographystyle{acm}

\end{document}